\documentclass[11pt,a4paper]{article}
\usepackage{fullpage}
\usepackage[utf8]{inputenc}

\usepackage{makeidx}  
\usepackage{paralist}
\usepackage{color, soul}
\usepackage{amsmath, amsfonts,amssymb}
\usepackage{graphicx}
\usepackage{float}
\usepackage{subfig}
\usepackage{tabularx}

\usepackage{url}
\usepackage{framed}
\usepackage{amsthm}
\newtheorem{proposition}{Proposition}
\usepackage[breaklinks,bookmarks=false]{hyperref}

\newif\ifdraft
\drafttrue

\ifdraft
 \definecolor{jtgreen}{RGB}{20,150,60}
 \newcommand{\PF}[1]{{\color{red}{\bf pf: #1}}}
 
 \newcommand{\AM}[1]{{\color{blue}{\bf am: #1}}}
 
 \newcommand{\JT}[1]{{\color{jtgreen}{\bf jt: #1}}}
 
\else
 \newcommand{\PF}[1]{}
 
 \newcommand{\AM}[1]{}
 
 \newcommand{\JT}[1]{}
 
\fi

\newcommand{\bE}[0]{\mathcal{E}}
\newcommand{\bG}[0]{\mathcal{G}}

\newcommand{\bR}[0]{\mathcal{R}}

\newcommand{\bW}[0]{\mathbf{W}}

\newcommand{\bx}[0]{\mathbf{x}}

\newcommand{\DPS}[0]{\textbf{DPPS}}
\newcommand{\RS}[0]{\textbf{RS}}

\newcommand{\US}[0]{\textbf{US}}

\newcommand{\EMOC}[0]{\textbf{EMOC}}

\newcommand{\MSTP}[0]{\textbf{MSTP}}
\newcommand{\QMIP}[0]{\textbf{MIP}}

\newcommand{\comment}[1]{}

\newcommand{\keywords}[1]{~\\ \noindent \textbf{Keywords:} #1}

\begin{document}

\title{Active Learning and Proofreading \\ for Delineation of Curvilinear Structures}
\author{
Agata Mosinska\thanks{Supported by the Swiss National Science Foundation.} \\ EPFL
\and
Jakub Tarnawski\thanks{Supported by ERC Starting Grant 335288-OptApprox.}  \\ EPFL
\and
Pascal Fua \\ EPFL 
\and
{\tt\small  \{agata.mosinska, jakub.tarnawski, pascal.fua\}@epfl.ch} 
}

\clearpage\maketitle
\thispagestyle{empty}

\begin{abstract}

Many state-of-the-art delineation methods rely on supervised machine learning algorithms. As a result, they require manually annotated training data, which is tedious to obtain.  Furthermore,  even minor classification errors may significantly affect the topology of the final result. In this paper we propose a generic approach to addressing both of these problems by taking into account the influence of a potential misclassification on the resulting delineation. In an Active Learning context, we identify parts of linear structures that should be annotated first in order to train a classifier effectively. In a proofreading context, we similarly find regions of the resulting reconstruction that should be verified in priority to obtain a nearly-perfect result. In both cases, by focusing the attention of the human expert on potential classification mistakes which are the most critical parts of the delineation, we reduce the amount of required supervision. We demonstrate the effectiveness of our approach on microscopy images depicting blood vessels and neurons.

\keywords{Active Learning, Proofreading, Delineation, Light Microscopy, Mixed Integer Programming}

\end{abstract}


\section{Introduction}

Complex and extensive curvilinear structures include blood
vessels, pulmonary  bronchi, nerve fibers  and neuronal networks  among others.
Many  state-of-the-art  approaches to  automatically  delineating  them rely  on
supervised  Machine Learning  techniques.  For  training purposes,  they require
\textit{annotated} ground-truth data  in large quantities to cover  a wide range
of  potential variations  due to  imaging artifacts  and changes  in acquisition
protocols.  For optimal performance, these variations must be featured
in the training  data, as they can produce drastic changes in appearance. Furthermore, no  matter how well-trained the algorithms are,  they will  continue to  make mistakes,  which must  be
caught by the user and corrected.  This is known as  {\it proofreading} -- a slow,
tedious  and expensive  process when  large amounts  of image  data or  3D image
stacks are involved,  to the point that  it is considered as  a major bottleneck
for applications such as neuron reconstruction~\cite{Peng11c}.


\begin{figure*}[]
\centering
\subfloat[]{\includegraphics[height=0.18\textwidth]{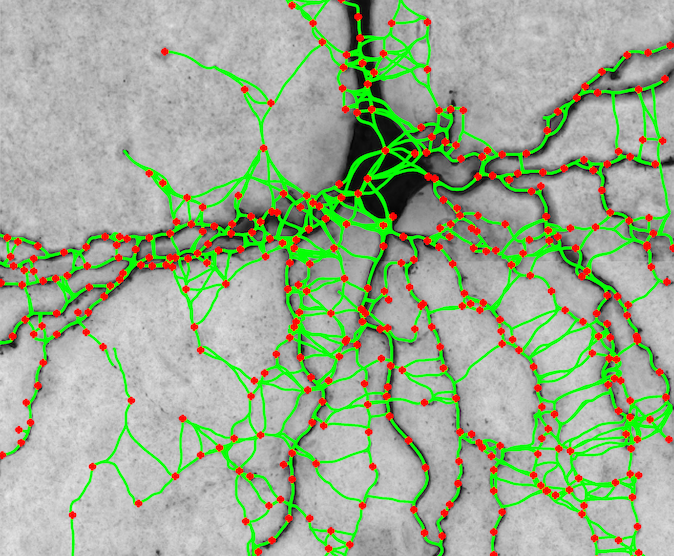}}
\hspace{0.1cm}
\subfloat[]{\includegraphics[height=0.18\textwidth]{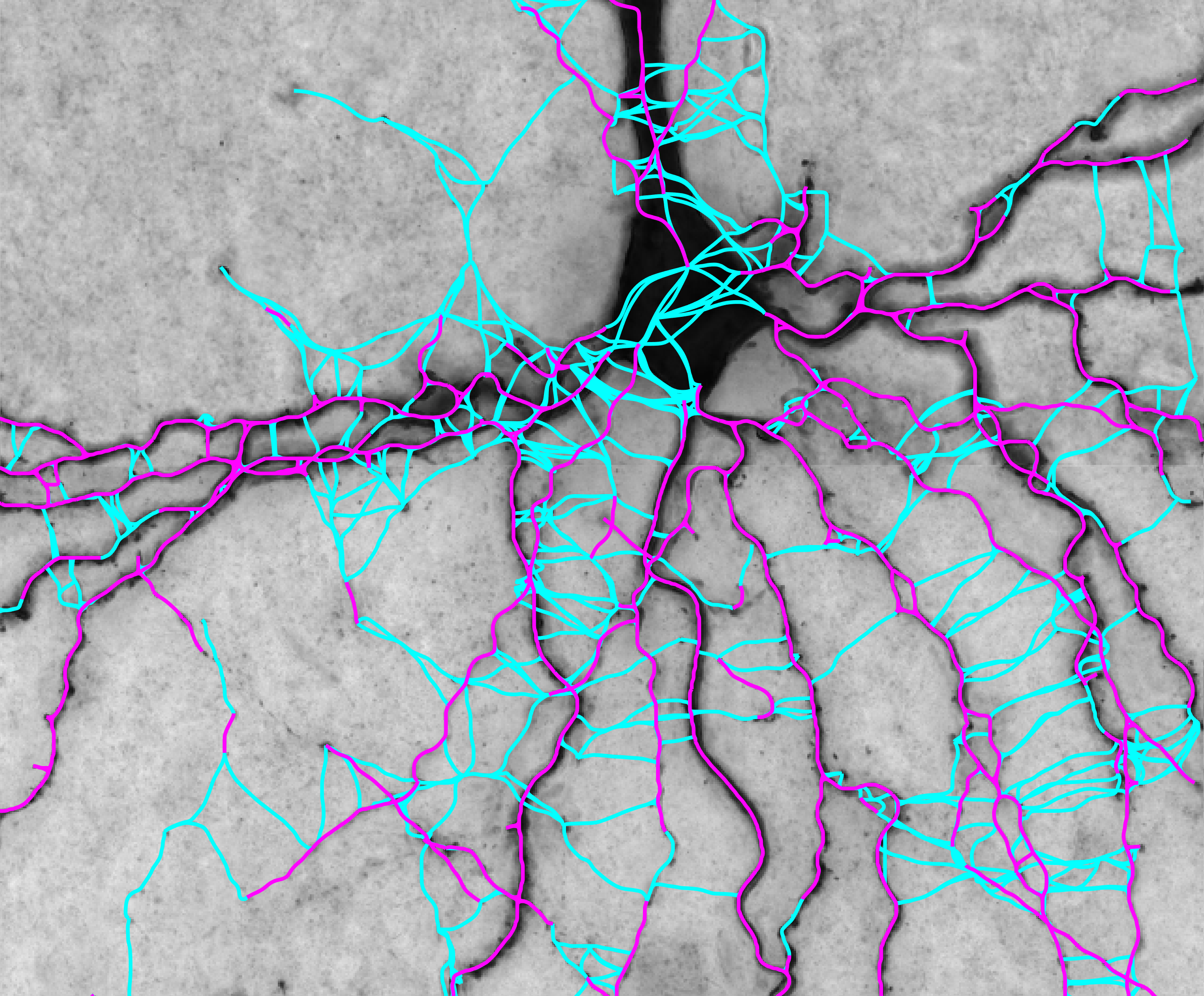}}
\hspace{0.1cm}
\subfloat[]{\includegraphics[height=0.18\textwidth]{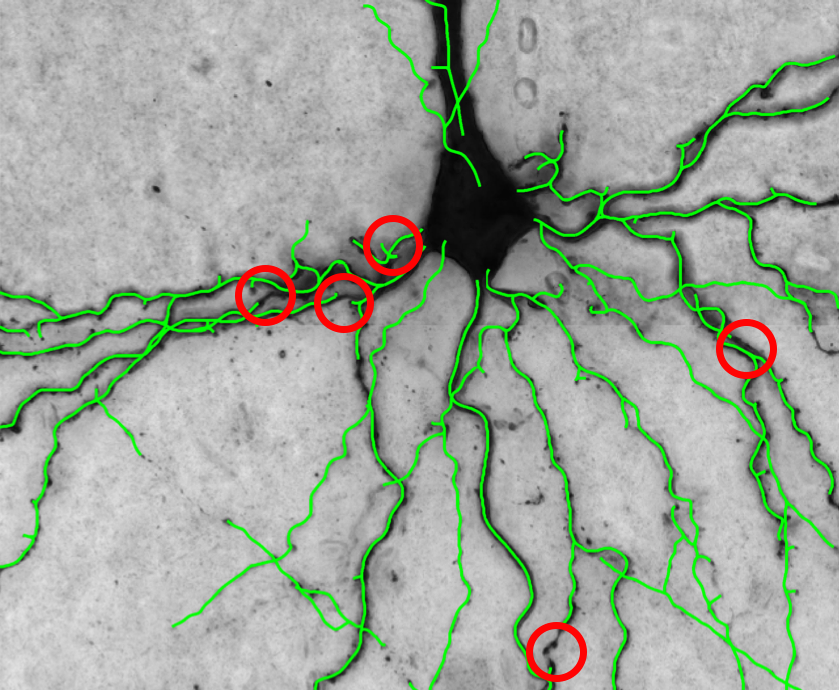}}
\hspace{0.1cm}
\subfloat[]{\includegraphics[height=0.18\textwidth]{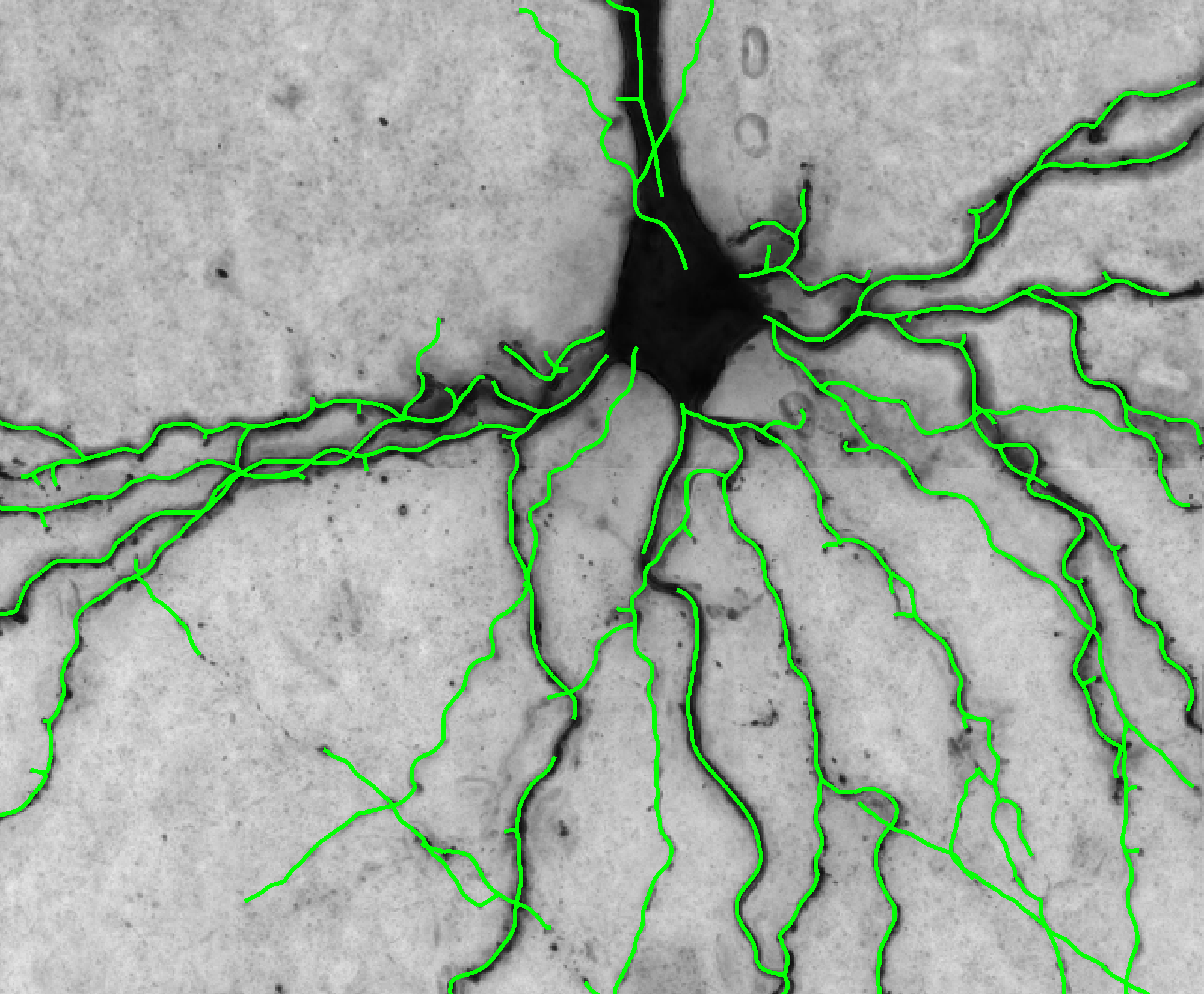}}
\vspace{-4mm}
\caption{Delineation  workflow.    (a)  Input  image  with   overcomplete  graph
  overlaid.  (b)  The
  high-probability  edges are  shown  in purple  and the  others  in cyan.   (c)~Automated delineation, with connectivity errors highlighted by red circles. (d)~Final result after proofreading.  All figures are best viewed in color.}
\label{fig:graph}
\vspace{-4mm}
\end{figure*}

\begin{figure}[]
  \centering
  \begin{tabular}{cc}
\includegraphics[height=0.23\textwidth]{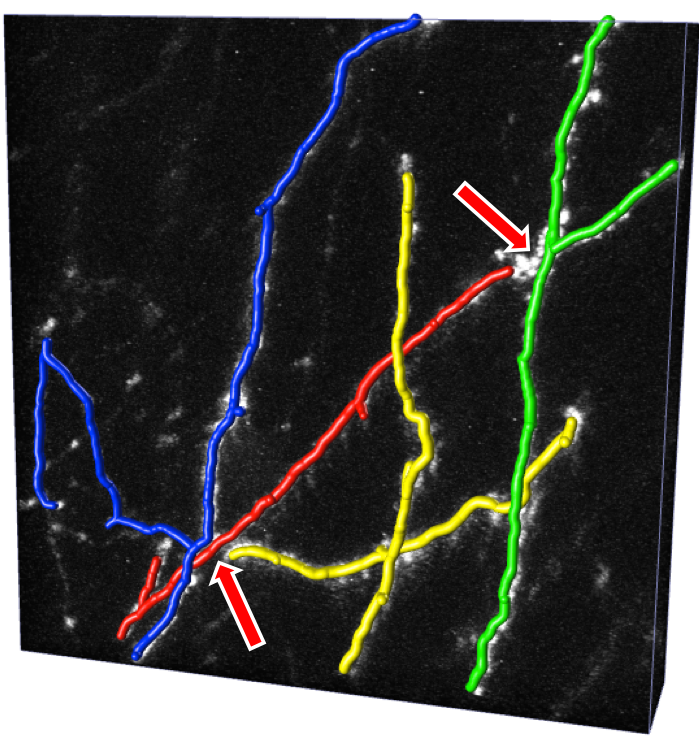}&
\includegraphics[height=0.23\textwidth]{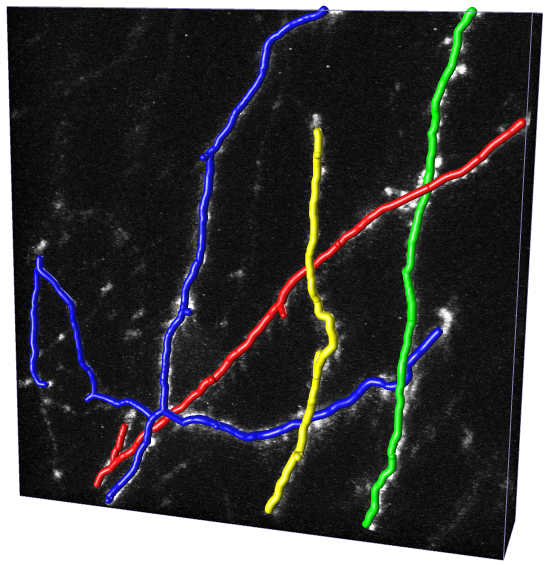}\\
(a)&(b)
  \end{tabular}
  \vspace{-4mm}
\caption{Misclassifying even a few edges may severely impact the final topology.
  (a)  The  two  edges indicated  by  the  red  arrows  are falsely  labeled  as
  negatives.  As  a result,  two pairs of unrelated branches (green and yellow) are merged. (b)~The true connectivity is recovered after correcting the two edges.}
\label{fig:mistakes}
\vspace{-4mm}
\end{figure}

In other  words, human  intervention is  required both  to create  training data
before   running  the   delineation  algorithm   and  to   correct  its   output
thereafter. Current  approaches to making  this less tedious focus  on providing
better  visualization  and   editing  tools  ~\cite{Dercksen14,Peng11c}.   While
undoubtedly useful, this is not enough.  We therefore propose an Active Learning
(AL)~\cite{Settles10} approach to direct the
annotator's attention to the  most critical samples. It takes into  account the expected  change in
reconstruction that can result from labeling specific paths. It can be  used both for fast
annotation purposes and, later, to  detect potential mistakes in machine-generated
delineations.

More     specifically,    consider     an     algorithm     such    as     those
of~\cite{Santamaria-Pang15,Montoya14,Turetken16a,Neher15,Peng11b},        whose
workflow is  depicted by  Fig.~\ref{fig:graph}.  It first  builds a  graph whose
nodes  are  points likely  to  lie  on the  linear  structures  and whose  edges
represent paths connecting them.  Then it assigns a weight to each edge based on
the  output of  a  discriminative  classifier. Since  the  result is  critically
dependent  on the  weights,  it  is important  that  the  classifier is  trained
well. Finally,  the reconstruction algorithm finds a  subgraph that  maximizes an objective (cost) function 
dependent on  the edge weights, subject to certain constraints. However,  even very small mistakes  can result in
very  different  delineations,  as  shown  in  Fig.~\ref{fig:mistakes}.  

Our main insight is  that the decision about which edges to annotate or proofread should be based on  their influence
on the cost of the network. Earlier methods either
  ignore the network topology  altogether~\cite{Freytag14} or only take  it into consideration
  locally~\cite{Mosinska16}, whereas we consider it globally. Our contribution
is therefore a cost- and topology-based criterion for detecting attention-worthy
edges.   We demonstrate  that this  can be  used for both AL  and proofreading,
allowing us to drastically reduce the required amount of human intervention 
  when used in conjunction with  the algorithm of~\cite{Turetken16a}. To make it
  practical for  interactive applications,  we also  reformulate the  latter to
  speed it up considerably
  -- it runs nearly in real-time and it can handle much larger graphs than~\cite{Turetken16a}.
  
The remainder of this paper is organized as follows.
First, in Section~\ref{sec:focus},
we describe our attention mechanism for selecting important edges in the delineation.
In Section~\ref{sec:active} we explain how this mechanism can be used for Active Learning and proofreading purposes.
Then, in Section~\ref{sec:mip}, we introduce a new, more efficient formulation of the state-of-the-art Mixed Integer Programming delineation algorithm that ensures fast and reliable reconstruction.
Finally, in Section~\ref{sec:results}, we compare the performance of our algorithm against conventional techniques.


\section{Attention Mechanism}
\label{sec:focus}

\subsection{Graph-Based Delineation}
\label{sec:delin}

Delineation algorithms usually start by computing a tubularity measure \cite{Law08,Turetken13c,Sironi16a}, which quantifies the likelihood  that a
tubular structure  is present  at a  given image  location.  Next,  they extract either
high-tubularity   superpixels  likely   to   be   tubular  structure   fragments
~\cite{Santamaria-Pang15,Montoya14} or  longer paths  connecting points
likely       to       be       on       the       centerline       of       such
structures~\cite{Gonzalez08,Breitenreicher13,Neher15,Turetken16a}.          Each
superpixel or path is treated as an edge $e_i$ of an over-complete spatial graph
$\bG$ (see Fig.~\ref{fig:graph}(a)) and is characterized by
an image-based feature vector $\bx_i$.  Let $\bE$  be the set of all such edges,
which is expected to  be a superset of the set $\bR$ of  edges defining the true
curvilinear  structure,  as  shown  in  Fig.~\ref{fig:graph}(d).  If  the  events of each edge $e_i$ being present in the reconstruction are assumed to be independent (conditional on the image evidence $\bx_i$), then the most likely subset $\bR^*$ is the
one minimizing
\begin{equation}
c(\bR) = \sum_{e_i\in \bR} w_i, \mbox{ with } w_i=-\log\frac{p(y_i = 1|\bx_i)}{p(y_i =
  0|\bx_i)} \; ,
\label{eq:ObjF}
\end{equation}
where $w_i  \in \bR$  is the  {\it weight} assigned  to edge  $e_i$ and $y_i$  is a
binary class label denoting whether $e_i$ belongs to the final reconstruction or
not.
This optimization is subject to certain geometric constraints;
for example, a state-of-the-art method presented  in~\cite{Turetken16a} solves a more  complex Mixed  Integer Program (\QMIP{}), which uses linear constraints to force the reconstruction to form a connected network (or a tree). As  described  in Section~1 of the  supplementary
material, we were  able to reformulate the original optimization  scheme and obtain major speedups which make it
practical even when delineations must be recomputed
often. There,  we also show that  it yields better
  results  than using  a  more basic method Minimum Spanning  Tree with  Pruning~\cite{Gonzalez08},
  while also being able to handle non-tree networks.
  Let us remark that finding the minimizing $\bR$ is trivial to parallelize.

The probabilities  appearing in Eq.~\ref{eq:ObjF} can  be estimated in
many ways.
A simple and effective one
is to  train a discriminative  classifier for
this purpose~\cite{Breitenreicher13,Montoya14,Turetken16a}.  However, the performance
critically depends on how well-trained the classifier  is.  A
few  misclassified edges  can produce  drastic topology  changes, affecting  the
whole  reconstruction, as  shown  in Fig.~\ref{fig:mistakes}.  In  this paper  we
address both issues with a single generic criterion.

\subsection{Error Detection}
\label{sec:error}

The key  to both fast proofreading  and efficient AL is to quickly
find potential mistakes, especially those that are critical for the topology. In this work, we take {\it critical mistakes} to mean erroneous
edge weights $w_i$ that result  in major changes to  the cost $c(\bR^*,\bW)$ of
the reconstruction. In other  words, if changing
a specific weight can significantly influence the delineation, we must ensure that the weight is correct. We
  therefore  measure this  influence, alter  the edge  weights accordingly,  and
  recompute the delineation.

\subsubsection{Delineation-Change Metric}
\label{sec:metric}

We  denote by  $\bR^*$ the  edge subset  minimizing  the objective (cost) function
$c(\bR, \bW) = \sum_{e_i \in \bR} w_i$ given a  particular set $\bW$ of weights assigned  to edges in
$\bG$.  Changing the weight $w_i$ of edge  $e_i$ to $w_i'$ will lead to a new graph
with optimal edge  subset  $\bR'_i$.  We  can  thus  define  a delineation-change metric, which evaluates the cost of changing the weight of an edge $e_i
\in \bE$:
\begin{equation}
  \Delta c_i = c(\bR^*,\bW) - c(\bR'_i,\bW') \; .
\label{eq:weightMetric}
\end{equation}
If $\Delta c_i >  0$, the cost has decreased; we can  conjecture that the overall
reconstruction benefits from this weight change and therefore the weight value may be worth investigating by the annotator as a potential mistake. The converse is true if $\Delta c_i
< 0$.  In other words,  this very simple  metric gives us  a way to  gauge the
influence of an edge weight on the overall reconstruction.

\subsubsection{Changing the Weights}
\label{sec:weights}

\begin{figure*}[]
\centering
\subfloat[]{\includegraphics[height=0.35\textwidth]{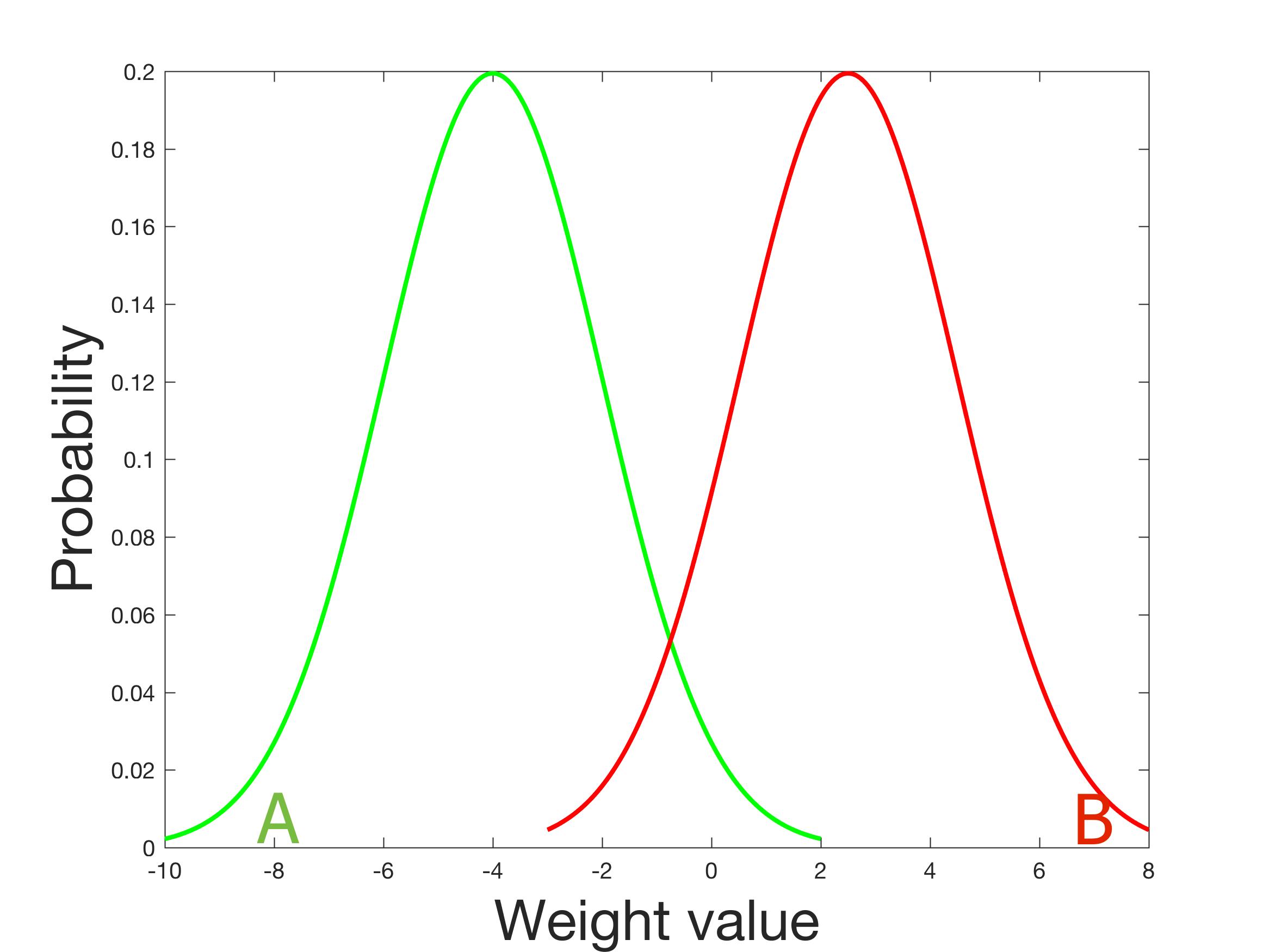}}
\subfloat[]{\includegraphics[height=0.35\textwidth]{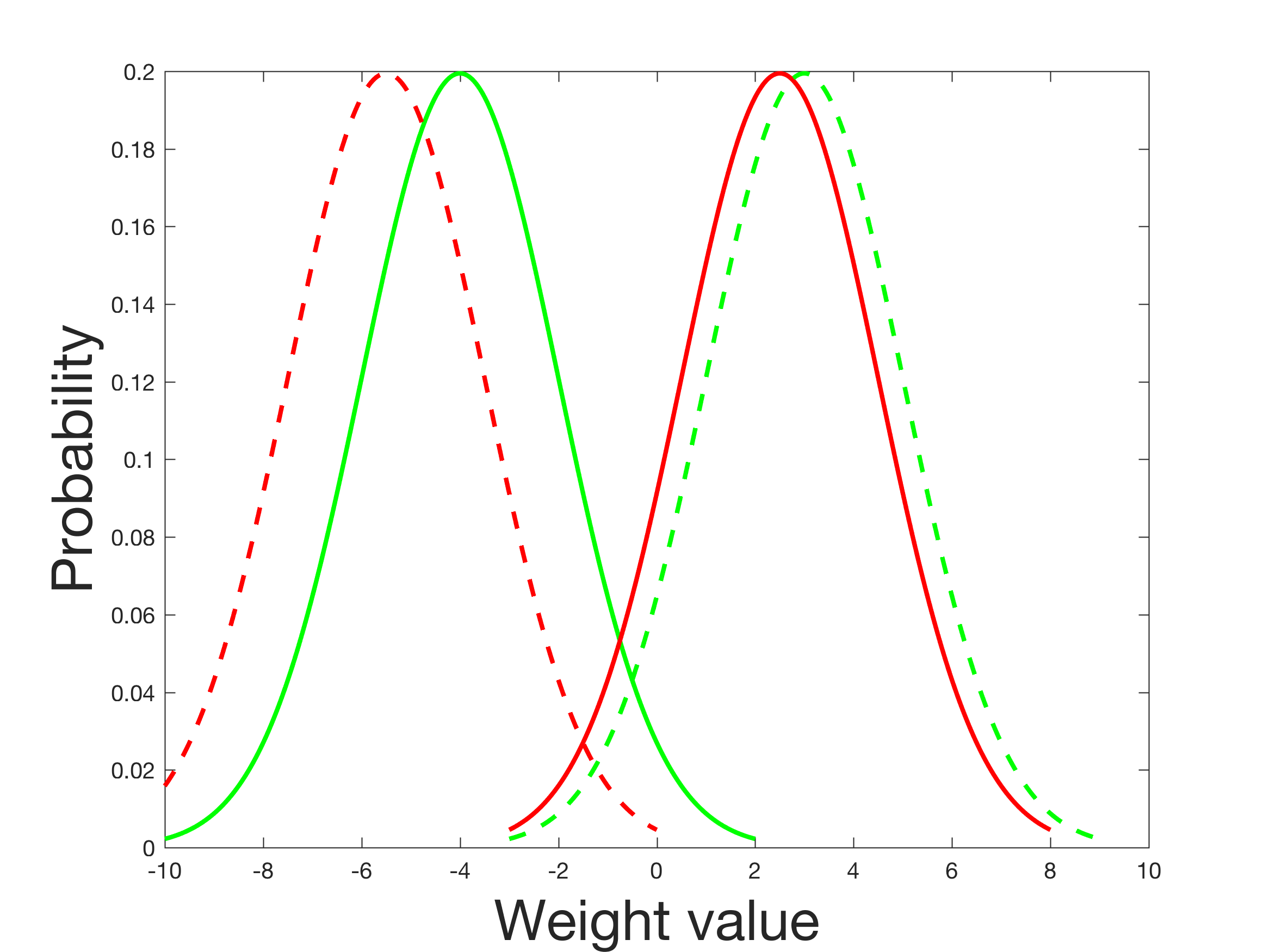}}
\caption{(a) Two Gaussian distributions corresponding to positive (green) and negative (red) classes of edges. (b) The effect of weight transformation; the original distributions are drawn with solid lines, while the corresponding distributions after the transformation are drawn with dashed lines. The described transformation causes ''swapping'' of the distributions corresponding to the two classes.}
\label{fig:WeightDistribution}
\end{figure*}

For  our  cost change criterion to  have practical  value, we  must
alter weights in such a way that $\Delta c_i$ is largest for edges which require
the opinion of an annotator. In practice, the weights of positive-class edges tend to follow a Gaussian distribution with negative mean and a variance such that few of them are positive values, as shown in Fig.~\ref{fig:WeightDistribution}(a). Similarly, negative edges follow a Gaussian distribution with positive mean, few of them being negative. As a result, most of the mistaken edges have $|w_i| \approx 0$.

 In order for our delineation-change metric to be informative, we must ensure that attention-worthy edges (probable mistakes) have high values of $\Delta c_i$. To achieve this, we must not only flip the sign of the weight (implying assigning it to the opposite class), but also increase the absolute value of likely mistakes. Without this, many of the mistakes with  $|w_i| \approx 0$ could be omitted due to smaller values of $\Delta c_i$ compared to edges with weights of higher absolute value, which are much less likely to be mistakes. 
 
The above requirements can be satisfied with the following transformation:
\begin{equation}
w'_i = \begin{cases} A + w_i & \mbox{if } w_i > 0, \\ B + w_i & \mbox{if } w_i
 < 0. \end{cases}
\label{eq:weightChange}
\end{equation}
It is equivalent to swapping the distributions corresponding to positive and negative edges, as shown in Fig~\ref{fig:WeightDistribution}(b).

We take $A$ and $B$ to  be the 10\% and
90\% quantiles of the weight distribution (for robustness to outliers). These are near-extreme values of the
weights for  the positive  and negative  classes respectively,  which we  use as
attractors for $w'_i$:  for small positive $w_i$  we want $w'_i$ to  be close to
$A$, and for negative ones to $B$ instead. The weight change is therefore likely to yield a significant $\Delta
c_i$ for probable mistakes. 
 
Finally, for  edges whose weight is  negative but which nevertheless do  not belong to
the  graph, we  take $\Delta  c_i$  to be  $w'_i$ to  ensure that it  is
positive and that more uncertain edges are assigned higher $\Delta c_i$.



\section{Active Learning and Proofreading}
\label{sec:active}

AL aims to train  a model with minimal user input by  selecting small subsets
of examples that are the most informative.
Formally, our  algorithm starts with a small set
of labeled edges $S_{0}$.  We then repeat the following steps: At iteration $t$,
we use  the annotated set of  edges $S_{t-1}$ to train  classifier $C_{t-1}$ and
select  one or  more  edges to  be labeled  by the  user  and added  to
$S_{t-1}$ to  form $S_{t}$. The  edge(s) we select  are those that  maximize the
criterion $\Delta c$ of Eq.~\ref{eq:weightMetric}.

By contrast,  proofreading occurs \emph{after}  the classifier  has been trained  and a
complete delineation has been produced.  At  this point, the main concern is not
to   further  improve   the  classifier,   but  simply   to  correct   potential
mistakes. Therefore, the most crucial edges are those that are misclassified and
whose presence or absence most affects the topology of the delineation.  To
find them, we again compute the $\Delta  c$ value for each edge.  However, some
edges could  have a high $\Delta  c$ because they are  misclassified, even though they do not
influence the topology of the final delineation.

To focus on potential mistakes that do  affect the topology strongly, we rely on
the  DIADEM  score~\cite{Ascoli10}, which captures  the  topological  differences
between trees, such as connectivity changes and  missing or spurious branches. It ranges
from 0 to  1; the larger the  score, the more similar the two  trees are. More
specifically, let  $\bR^*$ be the  optimal tree given  the edge weights,  and let
$\bR'_i$ be the tree we obtain when changing the weight of edge $e_i$ from $w_i$
to $w'_i$, as described in  Section~\ref{sec:weights}. To measure the importance
of each edge, we compute the score
\begin{equation}
  s_i = \frac{\Delta c_i}{\mathrm{DIADEM}(\bR^*,\bR'_i)}
  \label{eq:diadem}
\end{equation}
and ask the user to check the highest-scoring one. The edge is assigned a weight
equal  to $A$  or $B$  from  Section~\ref{sec:weights} according  to the  user's
response. We then recompute $\bR^*$ and repeat the process. Note that this  is very  different  from  traditional proofreading  approaches, which require the user to visually inspect the whole image. By contrast, our user only
has to  give an  opinion about one edge  at a time, which is automatically selected and presented  to them.

\vspace{1cm}

\section{Fast Reconstruction of Curvilinear Structures}
\label{sec:mip}

To  delineate networks  of  curvilinear  structures, we  rely  on the  algorithm
of~\cite{Turetken16a}, which involves solving the following problem:

\begin{framed}
\begin{center}\textbf{Min-Weight Tree Containing $r$ (MinTree)} \end{center}
\begin{description}
  \item[\textnormal{\emph{Given:}}]A graph $\bG = (V,\bE)$, a root vertex $r \in V$, weights on edges $w : \bE \to \mathbb{R}$. Weights may be negative.
	
\item[\textnormal{\emph{Find:}}] A tree $\bR \subseteq \bG$ containing the vertex $r$, minimizing the sum of weights of picked edges $\sum_{e \in \bR} w(e)$.
  \end{description}
\end{framed}

In our approach, MinTree is used when we expect the ground-truth image to be a tree.
If such an assumption is not realistic (loopy networks, such as blood vessels), then we are instead interested in the following problem MinSubgraph:

\begin{framed}
\begin{center}\textbf{Min-Weight Connected Subgraph Containing $r$ (MinSubgraph)} \end{center}
\begin{description}
  \item[\textnormal{\emph{Given:}}]A graph $\bG = (V,\bE)$, a root vertex $r \in V$, weights on edges $w : \bE \to \mathbb{R}$. Weights may be negative.
	
\item[\textnormal{\emph{Find:}}] A connected subgraph $\bR \subseteq \bG$ which contains the vertex $r$ (and is not necessarily a tree), minimizing the sum of weights of picked edges $\sum_{e \in \bR} w(e)$.
  \end{description}
\end{framed}

Both problems are significantly harder than the Minimum Spanning Tree problem, because $\bR$ does not need to connect the entire graph
and also the weights may be negative.
In fact, both problems are NP-complete; we demonstrate this later in Proposition~\ref{prop:nphard}.
In both~\cite{Turetken16a} and our approach they are solved using a Mixed Integer Programming (\QMIP{}) formulation, which is given as input to the Gurobi solver.\footnote{\cite{Turetken16a} also introduce a more advanced algorithm, which uses a formulation with quadratic weights, i.e., weights on pairs of adjacent edges, rather than a linear weight function; this makes the computational burden even heavier.}

However, the previously considered formulation (see the model Arbor-IP in \cite{Turetken16a} and also the model M-DG in~\cite{BlumCalvo15}) has $|V||\bE|$ variables and as many constraints.
This makes solving it costly for small graphs and impossible for larger ones.
Our contribution is a new, linear-size \QMIP{} model for this problem.

In Section~\ref{sec:our_formulation} we introduce our formulation and argue about its correctness.
In Section~\ref{sec:hardness} we prove the NP-hardness of the considered problems.
The major running time improvements that the new formulation brings about are measured in Section~\ref{sec:running_time}.

Let us mention in passing that Blum and Calvo \cite{BlumCalvo15} also propose a ``matheuristic'' approach to solving MinTree -- although with no optimality guarantees.

\subsection{Our Formulation} \label{sec:our_formulation}

First, we describe how to obtain a \QMIP{} for MinTree.
We replace each undirected edge with two directed edges, so as to work with a directed graph.
Our objective is to find a directed tree whose each edge is directed away from the root $r$ (a so-called $r$-arborescence).

We associate a binary variable $x_{uv} \in \{0,1\}$ with each directed edge $(u,v) \in \bE$, denoting the presence of the edge in the solution $\bR$. The first two linear constraints to consider are:
\begin{itemize}
	\item any vertex $v$ has at most one incoming edge ($r$ has none) (see equations (\ref{eq:at_most_one_incoming}--\ref{eq:at_most_one_incoming_root}) below),
	\item an edge $(u,v)$ can be in the solution only if $u$ has an incoming edge in the solution (or $u = r$) \eqref{eq:only_if_has_incoming}.
\end{itemize}
These conditions almost require the solution to be an $r$-arborescence, but not quite; namely, there can still appear directed cycles (possibly with some adjoined trees). One way to deal with this issue is to enforce that every non-isolated vertex is connected to the root; this can be done using network flows. The constraints in the previous formulation require that, for every $v$ with an incoming edge, there should exist a flow $\{f_e^v\}_{e \in \bE}$ of value $1$ from $r$ to $v$. However, this leads to a large program ($|V||\bE|$ variables).

Our way around this is to instead require the existence of a single flow $\{f_e\}_{e \in \bE}$ from the source vertex $r$ to some set of sinks. The main constraints are that:
\begin{itemize}
	\item for every vertex $v \ne r$, if $v$ has an incoming edge (i.e., $v$ is not an isolated vertex in the solution, but is spanned by $\bR$), then the inflow into $v$ is at least $1$ more than the outflow (otherwise it is greater or equal to the outflow) \eqref{eq:flow_conservation},
	\item $f$ is supported only on the support of $x$ (that is, the flow $f$ only uses edges which are used by the solution $\bR$) \eqref{eq:flow_only_on_x}.
\end{itemize}
Since $x$ has no edges into the root, neither does $f$. Thus $f$ is indeed a flow (within the $x$-subgraph) from the source $r$ to the sink set being the set of all active vertices.

We write down our \QMIP{} formulation below.
We use the following notation: $x(F) = \sum_{e \in F} x(e)$ for a subset $F \subseteq \bE$, $\delta^+(v)$ is the set of (directed) edges outgoing from vertex $v$, and $\delta^-(v)$ is the set of (directed) edges incoming into vertex $v$. Thus e.g. $f(\delta^+(v))$ is the total $f$-flow outgoing from vertex $v$.

\begin{framed}
\begin{alignat}{3}
\text{minimize} & & \sum_{(u,v) \in \bE} & w(u,v) x_{uv} & & \nonumber  \\
\text{subject to} & & x_{uv} &\in \{0,1\} & & \qquad \forall (u,v) \in \bE \nonumber \\
& & x(\delta^-(v)) &\le 1 & & \qquad \forall v \in V \setminus \{ r \} \label{eq:at_most_one_incoming} \\
& & x(\delta^-(r)) &= 0 & & \label{eq:at_most_one_incoming_root} \\
& & x_{uv} &\le x(\delta^-(u)) & & \qquad \forall (u,v) \in \bE, u \ne r \label{eq:only_if_has_incoming} \\
& & f(\delta^-(v)) - f(\delta^+(v)) &\ge x(\delta^-(v)) & & \qquad \forall v \in V \setminus \{r\} \label{eq:flow_conservation} \\
& & f_{uv} &\ge 0 & & \qquad \forall (u,v) \in \bE \nonumber \\
& & f_{uv} &\le (|V|-1) \cdot x_{uv} & & \qquad \forall (u,v) \in \bE. \label{eq:flow_only_on_x}
\end{alignat}
\end{framed}

The following proposition explains the correctness of our formulation.

\begin{proposition} \label{prop:soundness}
For any $\bR \subseteq \bE$, the corresponding vector $x \in \{0,1\}^{\bE}$ is feasible for the \QMIP{} formulation\footnote{More precisely, there exists $f \in \mathbb{R}_+^{\bE}$ such that $(x,f)$ is feasible for the \QMIP{} formulation, where $x$ is obtained from $\bR$ by directing all edges to point away from~$r$.} iff $\bR$ is a tree containing the root $r$.
\end{proposition}
\begin{proof}
($\Longrightarrow$)
By \eqref{eq:at_most_one_incoming}, edges $(u,v)$ with $x_{uv} = 1$ form a (directed) subgraph where every vertex has indegree at most 1.
It is not hard to see that each connected component of such a graph is either a tree or a cycle (possibly with adjoined trees); the cycle case is impossible if the component contains $r$ (by \eqref{eq:at_most_one_incoming_root}).
We show that actually there is no connected component except the one containing $r$.
Towards a contradiction suppose that $S \subseteq V \setminus \{r\}$ is such a component;
we will show that the flow conservation constraints \eqref{eq:flow_conservation} must be violated.
Denote by $\delta^+(S) = \{ (u,v) \in \bE : u \in S, v \not \in S \}$ the outgoing edges of $S$,
and by $\delta^-(S)$ the incoming edges.
We have $x(\delta^+(S)) = x(\delta^-(S)) = 0$ and thus, by \eqref{eq:flow_only_on_x},
$f(\delta^+(S)) = f(\delta^-(S)) = 0$.
However, by summing up \eqref{eq:flow_conservation} over $v \in S$ we get $f(\delta^-(S)) - f(\delta^+(S)) \ge \sum_{v \in S} x(\delta^-(v))$; the left side is $0$ but the right side is positive, a contradiction.\footnote{The observant reader will notice that the constraint \eqref{eq:only_if_has_incoming} is redundant. However, we keep it for clarity of exposition and because it makes solving the program faster in practice.}

($\Longleftarrow$)
It is easy to see that constraints (\ref{eq:at_most_one_incoming}--\ref{eq:only_if_has_incoming}) are satisfied by $x$.
To obtain the flow, we begin with $f = 0$.
Then, for each vertex $v$ with $x(\delta^-(v)) = 1$, we route $1$ unit of flow from $r$ to $v$ inside $\bR$ (that is, we only use edges $e$ with $x_e = 1$) and add that flow to $f$.
(This is possible since $\bR$ is connected.)
This way we will satisfy \eqref{eq:flow_conservation}.
Since the number of such vertices is at most $|V|-1$, any edge will hold at most $|V| - 1$ units of flow, thus satisfying \eqref{eq:flow_only_on_x}.
\end{proof}

So far we have discussed MinTree.
To get a formulation for MinSubgraph, one only needs to omit the constraint \eqref{eq:at_most_one_incoming} and adjust the constraint \eqref{eq:flow_only_on_x} to become $f_{uv} \le |\bE| \cdot x_{uv}$.
Then $x$ is obtained from $\bR$ by choosing any spanning tree of $\bR$ and orienting tree edges to point away from $r$ and non-tree edges arbitrarily.
In the proof of Proposition~\ref{prop:soundness} we route $x(\delta^-(v))$ units of flow (rather than $1$ unit) for each $v$
(now any edge holds at most $|\bE|$ units of flow).
These are the only changes.

\subsection{Hardness} \label{sec:hardness}

In this section we argue that our problems are extremely unlikely to be solvable in polynomial time.
This makes solving \QMIP{} formulations
using state-of-the-art solvers
one of the most natural and efficient methods available.

\begin{proposition} \label{prop:nphard}
The problems MinTree and MinSubgraph are NP-complete.
\end{proposition}
\begin{proof}
Clearly both are in NP.
We will show an NP-hardness reduction from the Steiner tree problem in graphs (STP), which is a well-known NP-hard problem.
An instance of STP consists of a graph $G = (V,E)$ with weights on edges $w : E \to \mathbb{R}_+$ and a set of terminal vertices $T \subseteq V$.
The objective is to find a minimum-weight tree in $G$ which connects the set $T$.
To obtain an instance of MinTree (or MinSubgraph) from STP,
we do the following for each $t \in T$: adjoin a new vertex $t'$ to $t$ using a new edge $(t,t')$ of weight $-M$,
where $M$ is a very large weight (say $M = 1 + \sum_{e \in E} |w(e)|$).
Then set the root $r$ to be any of these new vertices.

To see that an optimal solution of the MinTree instance corresponds to an optimal solution of the STP instance,
note that the former must necessarily contain all the new edges
(as we set their weight to be so low that it makes sense to select them even if it requires us to also select many positive-weight edges).
Since the MinTree solution must be connected, it will therefore connect all the terminal vertices;
removing the new edges from the MinTree solution gives an optimal STP solution.
(The same reduction also works for MinSubgraph, since 
the weights of all original edges are positive
and thus
the optimal solution for MinSubgraph is the same as the optimal solution for MinTree.)
\end{proof}


\section{Results}
\label{sec:results}

\begin{figure*}[]
\centering
\subfloat[]{\includegraphics[height=0.23\textwidth]{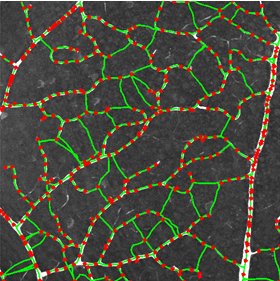}}
\hspace{0.05cm}
\subfloat[]{\includegraphics[height=0.23\textwidth]{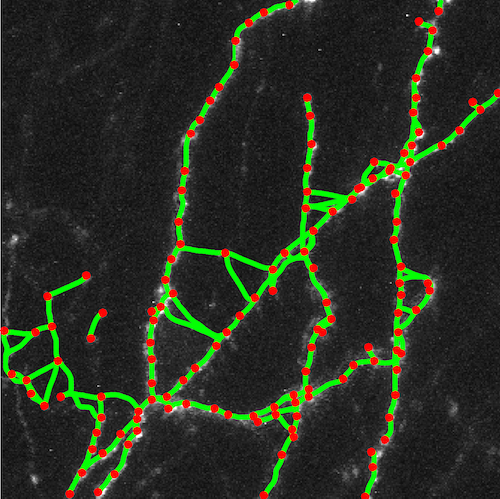}}
\hspace{0.05cm}
\subfloat[]{\includegraphics[height=0.23\textwidth]{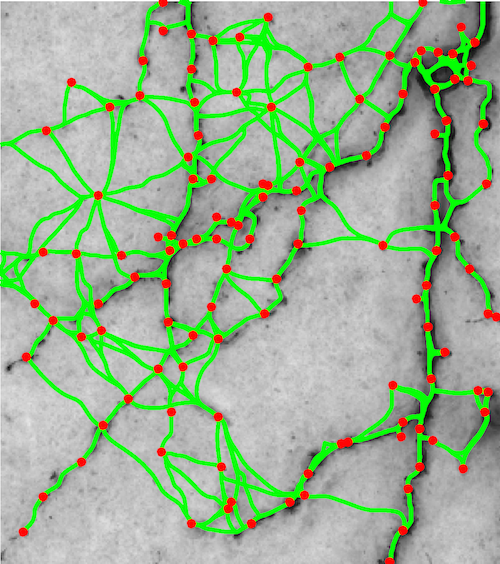}}
\hspace{0.05cm}
\subfloat[]{\includegraphics[height=0.23\textwidth]{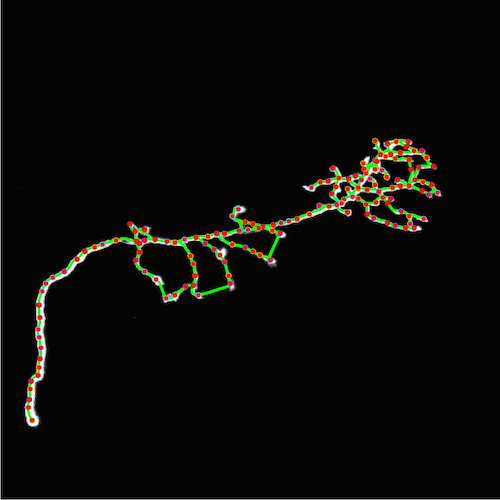}}
\vspace{-4mm}
\caption{Dataset   images   with   the   over-complete   graphs overlaid.   (a)
  \textit{Blood Vessels.} (b) \textit{Axons.} (c) \textit{Brightfield Neurons}. (d) \textit{Olfactory Projection Fibers.}}
\label{fig:datasets}
\end{figure*}
We tested our approach on 3-D  image stacks depicting retinal blood vessels, rat
brain axons and  dendrites, and drosophila olfactory  projection fibers obtained
using either 2-photon or brighfield microscopes, shown in Fig.~\ref{fig:datasets}.
We  rely on  the algorithm  of~\cite{Turetken16a} for  the initial  overcomplete
graphs, the  corresponding edge features  and the final delineations. To classify edges as being likely to be part  of an extended linear
structure or  not on  the basis of  local image  evidence, we  use Gradient
Boosted  Decision Trees~\cite{Becker13b}.

\subsection{Fast Reconstruction}
\label{sec:running_time}

The runtimes of our formulation compared to the one presented in~\cite{Turetken16a} are shown in Table~\ref{table:speed-up}.
The optimization was executed on a 2x Intel E5-2680 v2 system (20 cores). Our formulation can be solved under 6 seconds for all real-world graph examples we have tried; the maximum for the formulation of~\cite{Turetken16a} is over 6 minutes.

\begin{table}
\centering
\begin{tabular}{|c|c|c|c|c|c|c|}
 \hline
  & \textit{Axons1} & \textit{Axons2} & \textit{Axons3} & \textit{Axons4} & \textit{Axons5} & \textit{Axons6}\\
  \hline
  \# edges & 164 & 223 & 224 & 265  & 932 & 2638 \\
  \hline
  \QMIP{} \cite{Turetken16a} & 0.91 & 1.04 & 1.19 & 1.45 & 78.3 & 393.7\\
  \hline
  \QMIP{} ours & 0.03 & 0.10 & 0.04 & 0.23 & 0.10 & 5.23 \\
  \hline
  speedup & 26.1x & 10.1x & 27.3x & 6.3x & 743.5 & 75.2x \\
  \hline
\end{tabular}

\begin{tabular}{|c|c|c|c|c|c|c|}
 \hline
  & \textit{BFNeuron1} & \textit{BFNeuron2} &\textit{OPF1}& \textit{OPF2}& \textit{BFNeuron3} & \textit{BFNeuron4}  \\
  \hline
  \# edges & 120 & 338 & 363 & 380 & 645 & 2826 \\
  \hline
  \QMIP{} \cite{Turetken16a}  &  0.48 & 2.25 &1.53 & 1.65& 2.13 & 308.23 \\
  \hline
  \QMIP{} ours  & 0.02 & 0.12 & 0.05 & 0.08& 0.26 & 2.30  \\
  \hline
  speedup & 18.2x & 17.7x & 29.4x & 19.9x& 8.1x  & 134.0x \\
  \hline
\end{tabular}
\caption{Per-reconstruction runtimes (in seconds) of the \QMIP{} formulation of~\cite{Turetken16a} and ours for the proofreading task.}
\label{table:speed-up}
\end{table}

We also compared the runtimes on randomly generated graphs of various sizes -- see Table~\ref{table:random}. The speed-ups remain similar. In Table~\ref{table:random2} we collect runtimes of our method on larger randomly generated graphs.
If we assumed (more or less arbitrarily) 2 seconds to be the threshold of what is practical in an interactive setting (given that this optimization needs to be run multiple times), then
we can see that the method of~\cite{Turetken16a} can deal with graphs of size at most 300, whereas our method copes with graphs having around 2000 edges.

\begin{table}
\centering
\begin{tabular}{|c|c|c|c|c|c|c|c|c|c|} 
  \hline
  \# edges & 99 & 132 & 220 & 330 & 440 & 660 & 924 & 1320 & 1540 \\
  \hline
  \QMIP{} \cite{Turetken16a} & 0.16 & 0.30 & 1.13 & 3.39 & 8.35 & 29.35 & 73.16 & 112.59 & 149.01 \\
  \hline
  \QMIP{} ours & 0.03 & 0.04 & 0.06 & 0.12 & 0.15 & 0.29 & 0.36 & 0.67 & 0.42 \\
  \hline
  speedup & 6.1x & 7.5x & 17.9x & 29.4x & 53.9x & 102.8x & 201.8x & 167.8x & 348.1x \\
  \hline
\end{tabular}
\caption{Per-reconstruction runtimes (in seconds) of the \QMIP{} formulation of~\cite{Turetken16a} and ours on random graphs.}
\label{table:random}
\end{table}

\begin{table}
\centering
\begin{tabular}{|c|c|c|c|c|c|c|}
  \hline
  \# edges & 1760 & 2420 & 3520 & 4400 & 5720 & 9900 \\
  \hline
  \QMIP{} ours & 1.60 & 2.71 & 6.59 & 9.57 & 15.52 & 81.55 \\
  \hline
\end{tabular}
\caption{Per-reconstruction runtimes (in seconds) of our \QMIP{} formulation on random graphs.}
\label{table:random2}
\end{table}

One further practical method for speeding up the solver is to initialize it with a nonzero feasible solution. In cases where we needed to explore a large number of reconstructions resulting from altering just one weight at a time (which was the setting of our paper), we initialized the new solution to the current optimal solution.
Note that this scenario makes performance considerations especially relevant,
as $|\bE|$ reconstructions need to be made;
even though they can be run in parallel,
a high running time of a single \QMIP{} solution would make the approach impractical.

\subsection{Active Learning}

\begin{figure*}[t!]
\centering
\subfloat[]{\includegraphics[width=0.42\textwidth]{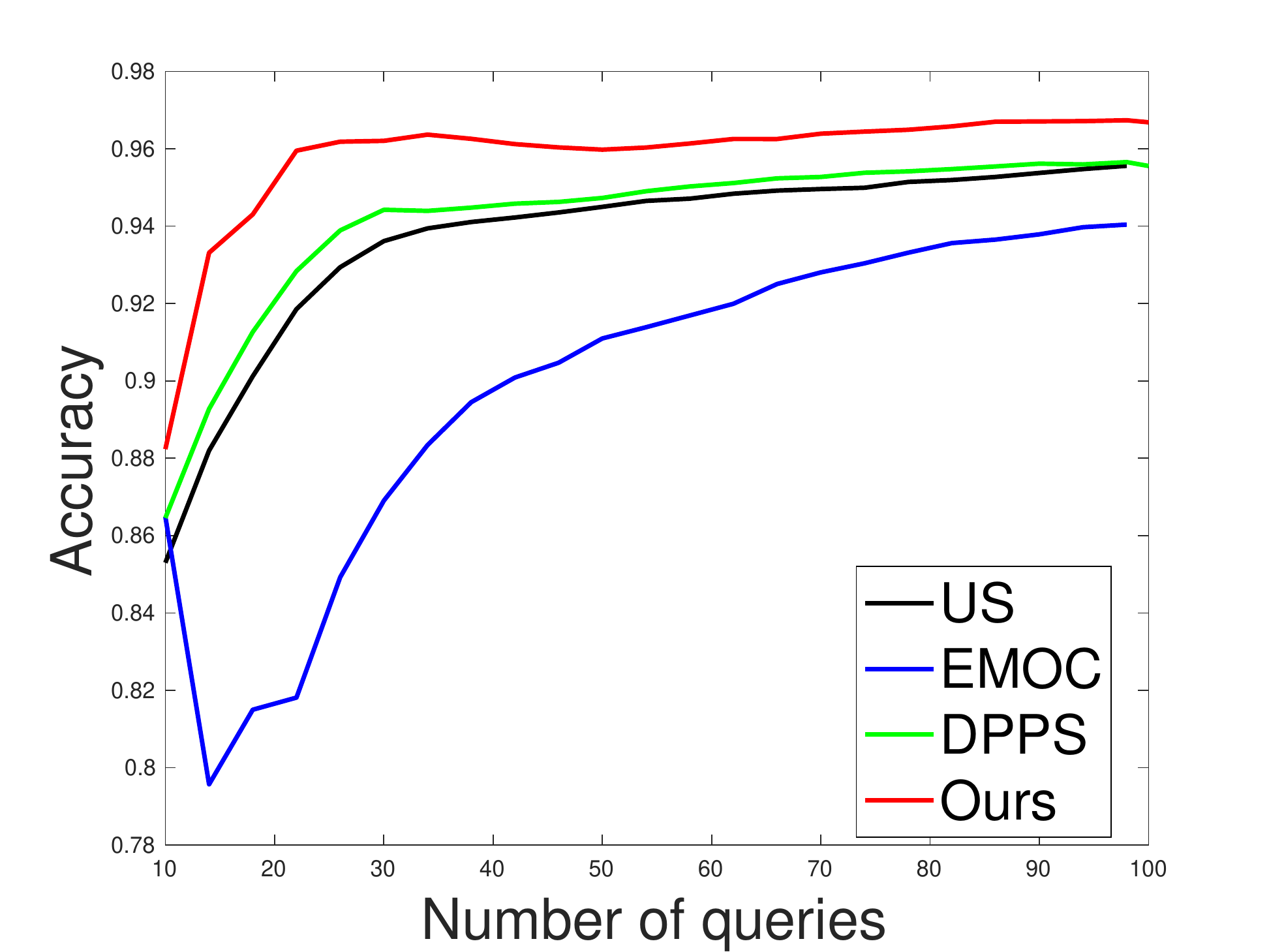}}
\subfloat[]{\includegraphics[width=0.42\textwidth]{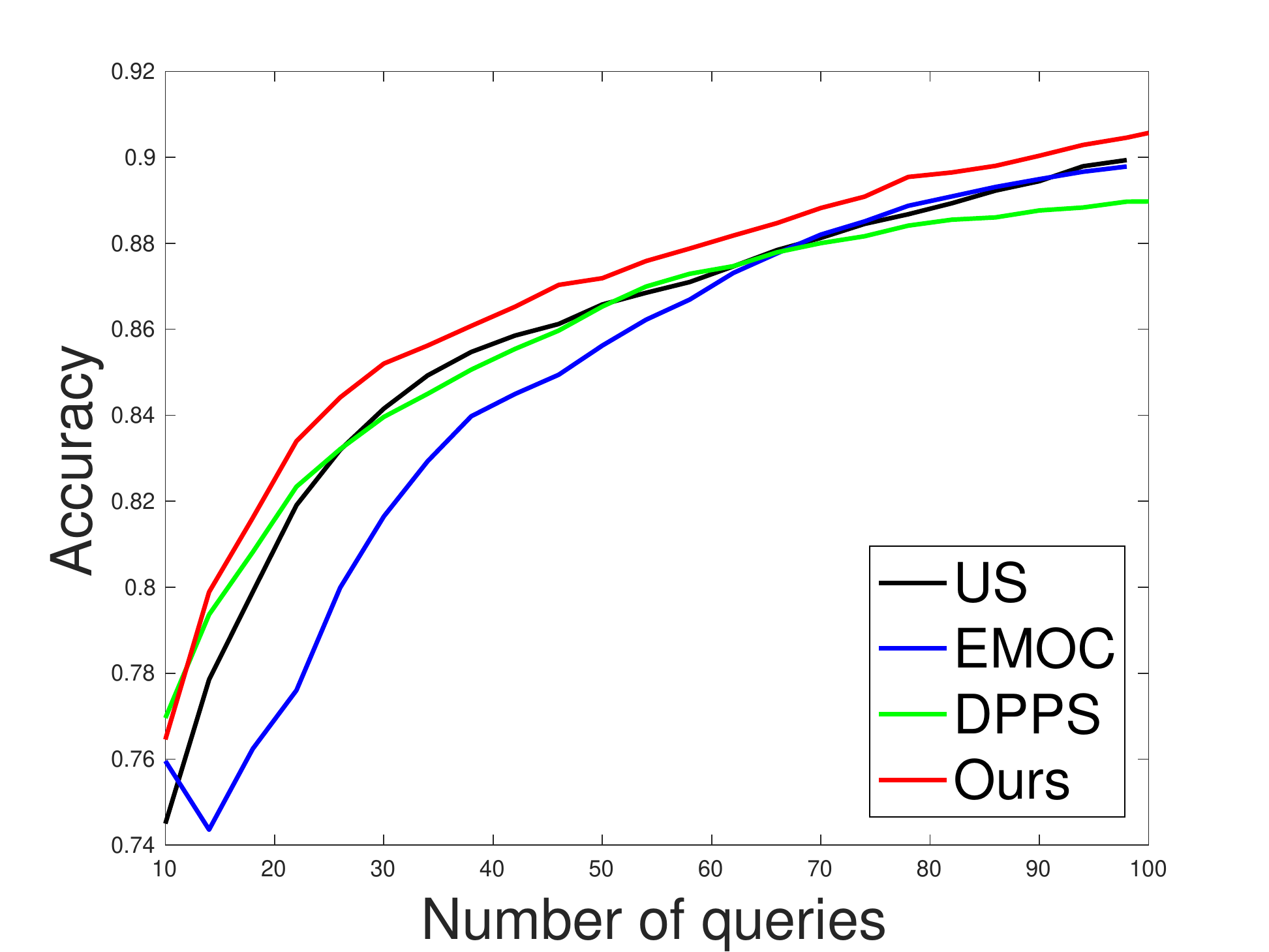}} \\  \vspace{-0.4cm}
\subfloat[]{\includegraphics[width=0.42\textwidth]{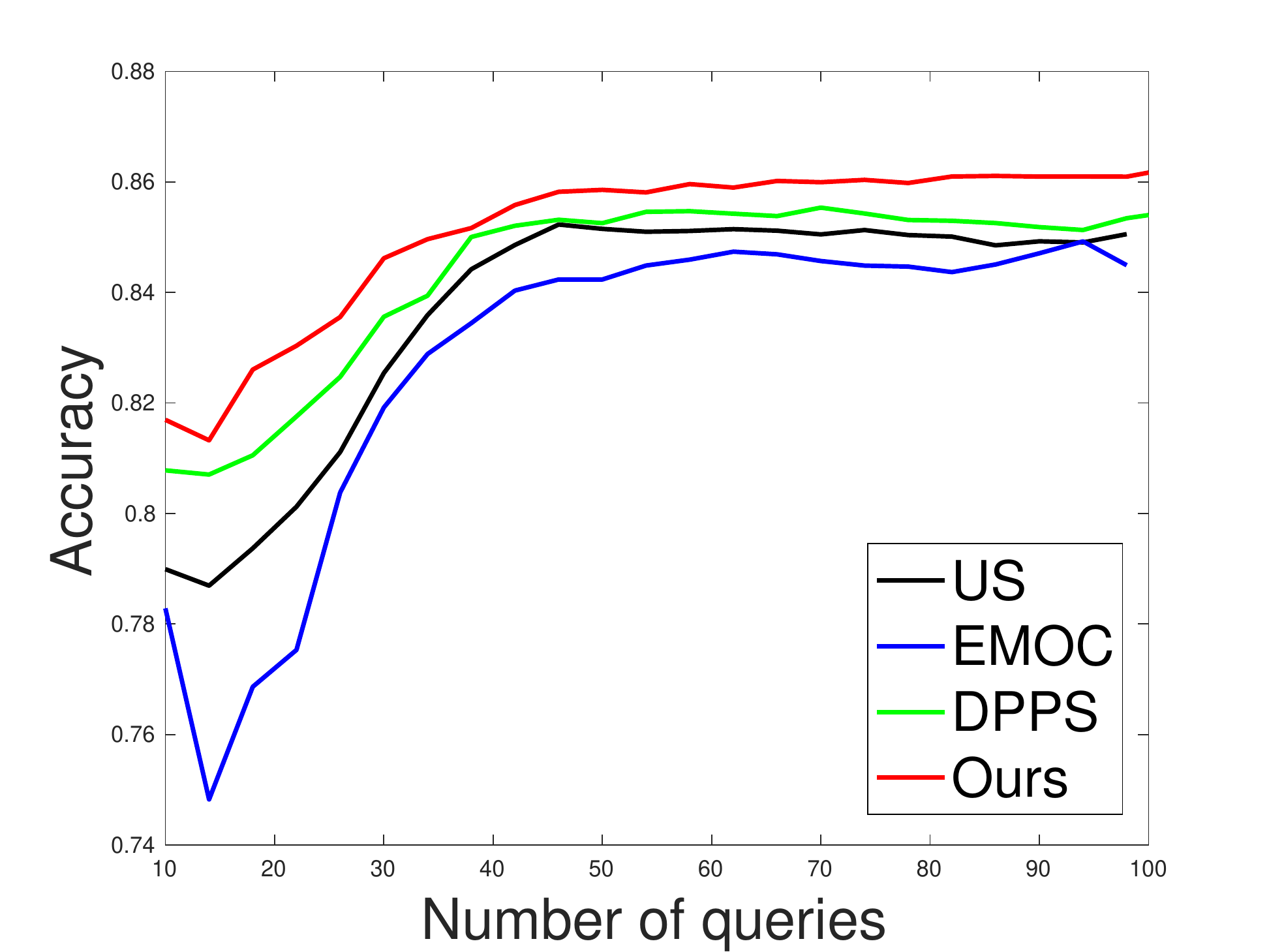}}
\subfloat[]{\includegraphics[width=0.42\textwidth]{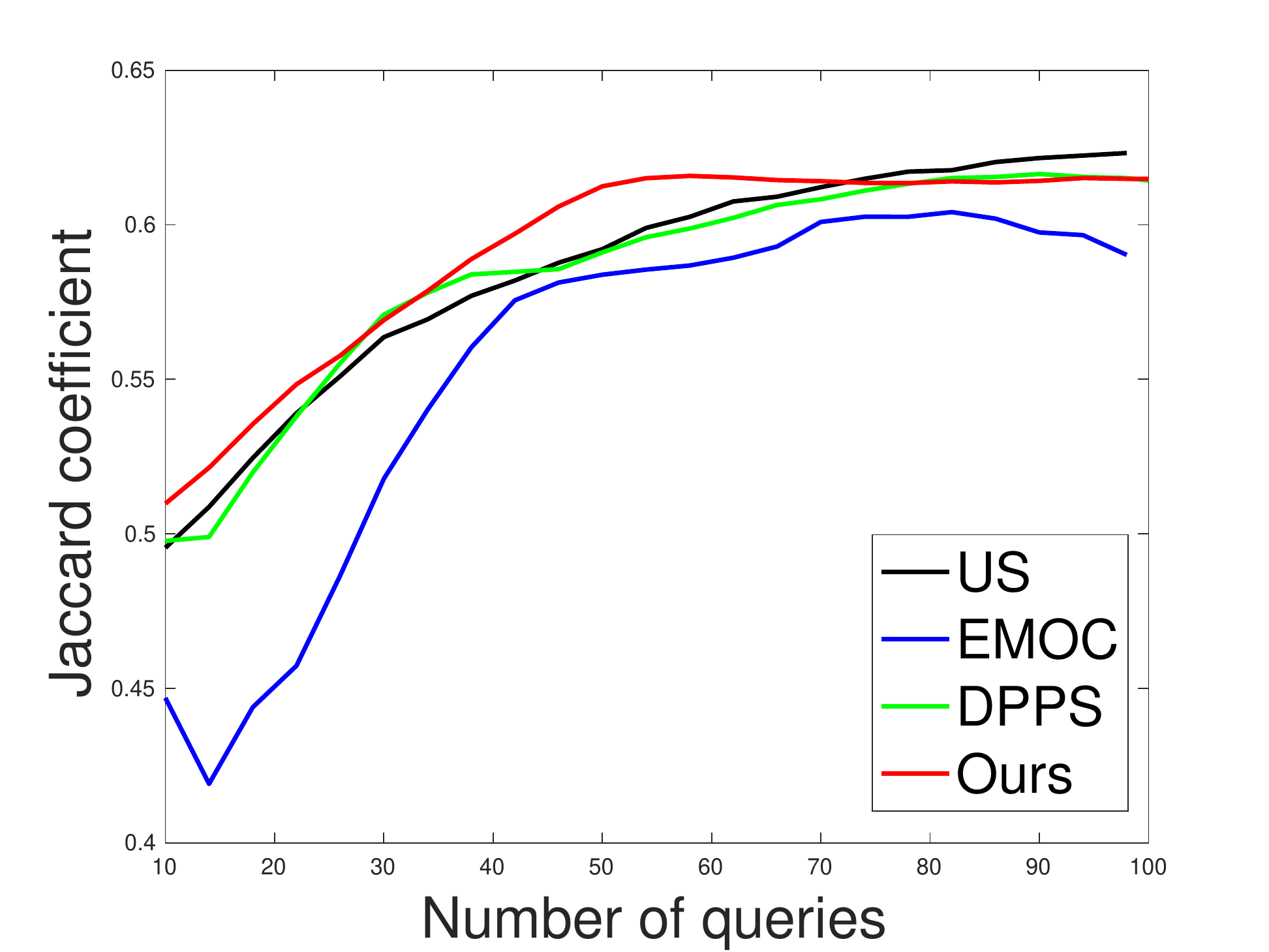}}
\caption{Active  Learning. Accuracy  as a  function of  the number  of annotated
  samples.     (a)~\textit{Blood    vessels.}    (b)~\textit{Axons.}    (c)
  \textit{Brightfield  neurons}.   (d)   \textit{Olfactory  Projection  Fibers.}
  The red curve denoting our approach is always above the others, except
    in the right-hand side of (d): because this is a comparatively easy case, the
    delineation  stops changing  after some time and  error-based queries  are no
    longer informative.}
\label{fig:ALresults}
\end{figure*}

For each image, we start with  an overcomplete graph.  The initial classifier is
trained using  10 randomly  sampled examples.   Then, we query  four edges  at a
time, as  discussed in Section~\ref{sec:active},  which allows us to  update the
classifier often enough  while decreasing the computational cost.  We report     results     averaged    over     30    trials     in
Fig.~\ref{fig:ALresults}. Our  approach outperforms  both naive methods  such as
Uncertainty  Sampling  (\US{})  and  more  sophisticated  recent  ones  such  as
\DPS{}~\cite{Mosinska16}  and  \EMOC{}~\cite{Freytag14}.    \DPS{}  is  designed
specifically for delineation  and also relies on uncertainty  sampling, but only
takes local topology into account  when evaluating this uncertainty.  \EMOC{} is
a more  generic method  that aims  at selecting samples  that have  the greatest
potential to change the output.

In Fig.~\ref{fig:MSTvsQMIP} we can see that using \QMIP{} formulations indeed helps improve the AL results, compared to a more basic method Minimum Spanning Tree with Pruning~\cite{Gonzalez08} (\MSTP{}), as it produces more accurate reconstructions and thus we can more reliably detect mistakes. This is visible especially in case of \textit{Blood Vessels}, which in reality can form loops. Those can be reconstructed using MinSubgraph \QMIP{}, but not with \MSTP{}.

\begin{figure*}[]
\centering
\subfloat[]{\includegraphics[height=0.33\textwidth]{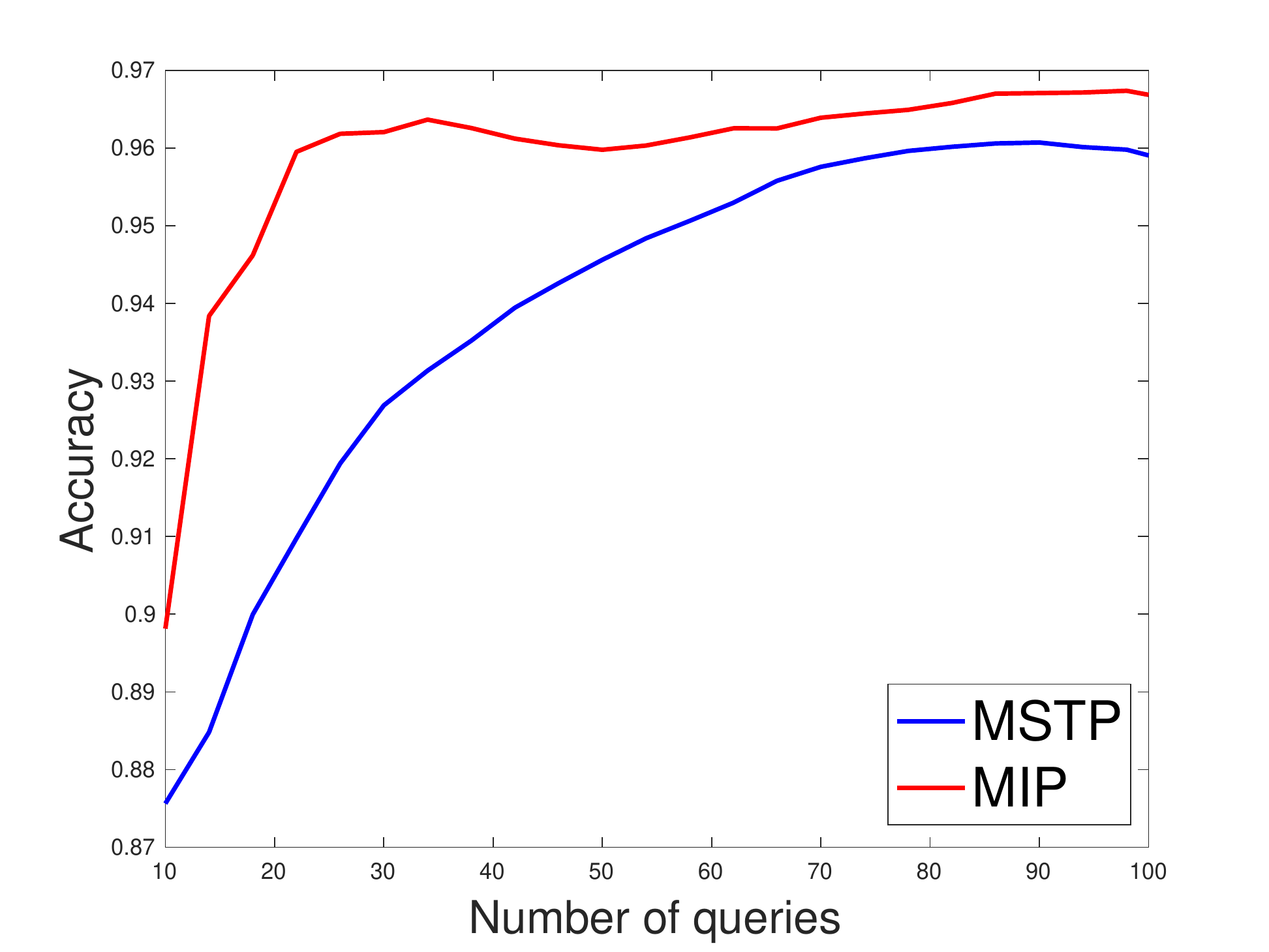}}
\subfloat[]{\includegraphics[height=0.33\textwidth]{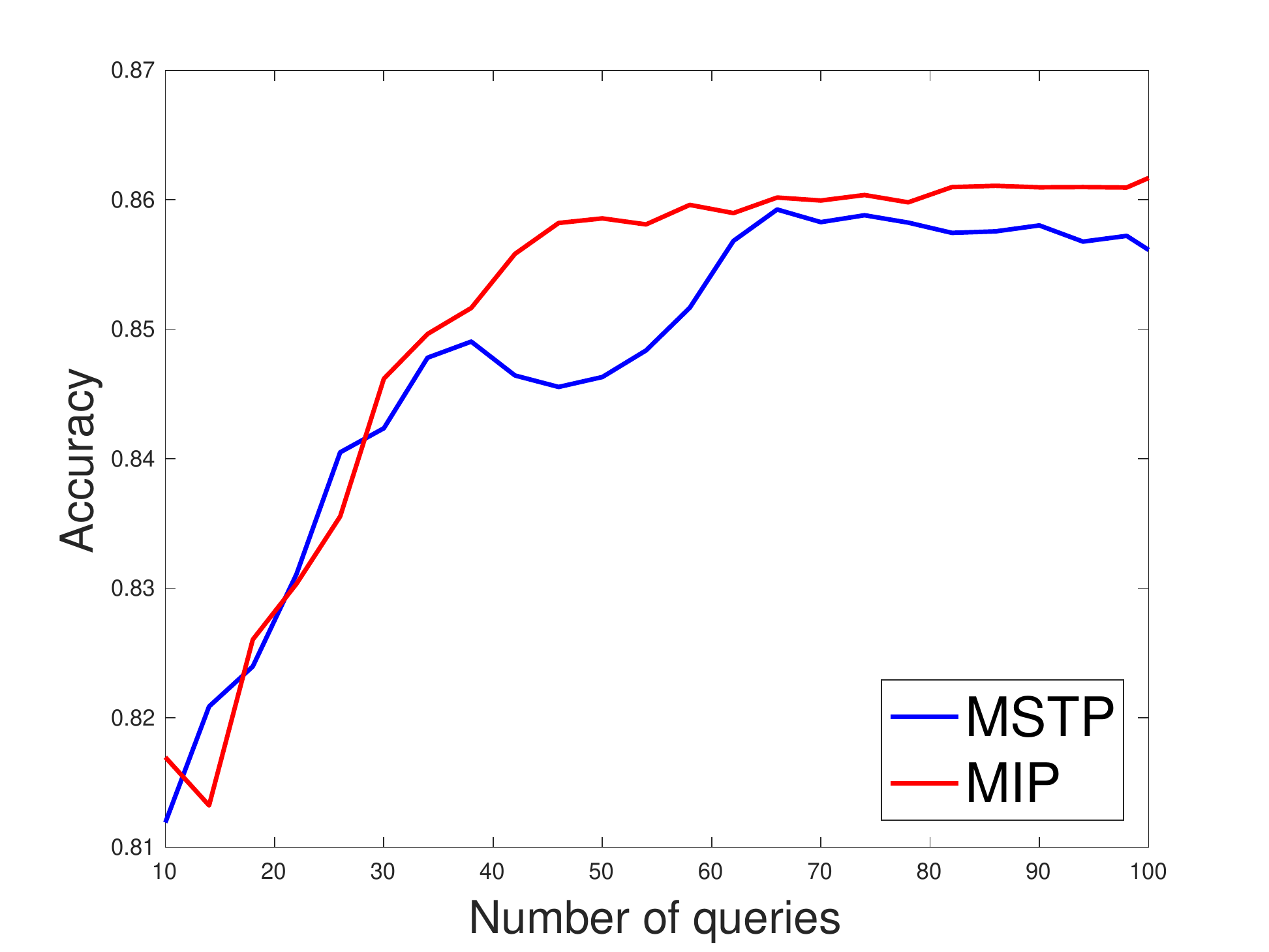}}
\caption{Comparison of our AL strategy when using \MSTP{} and \QMIP{}. (a)
  \textit{Blood Vessels.} (b) \textit{Brightfield Neurons.}} 
\label{fig:MSTvsQMIP}
\end{figure*}

\subsection{Proofreading}

\begin{figure*}[]
\centering
\subfloat[]{\includegraphics[height=0.33\textwidth]{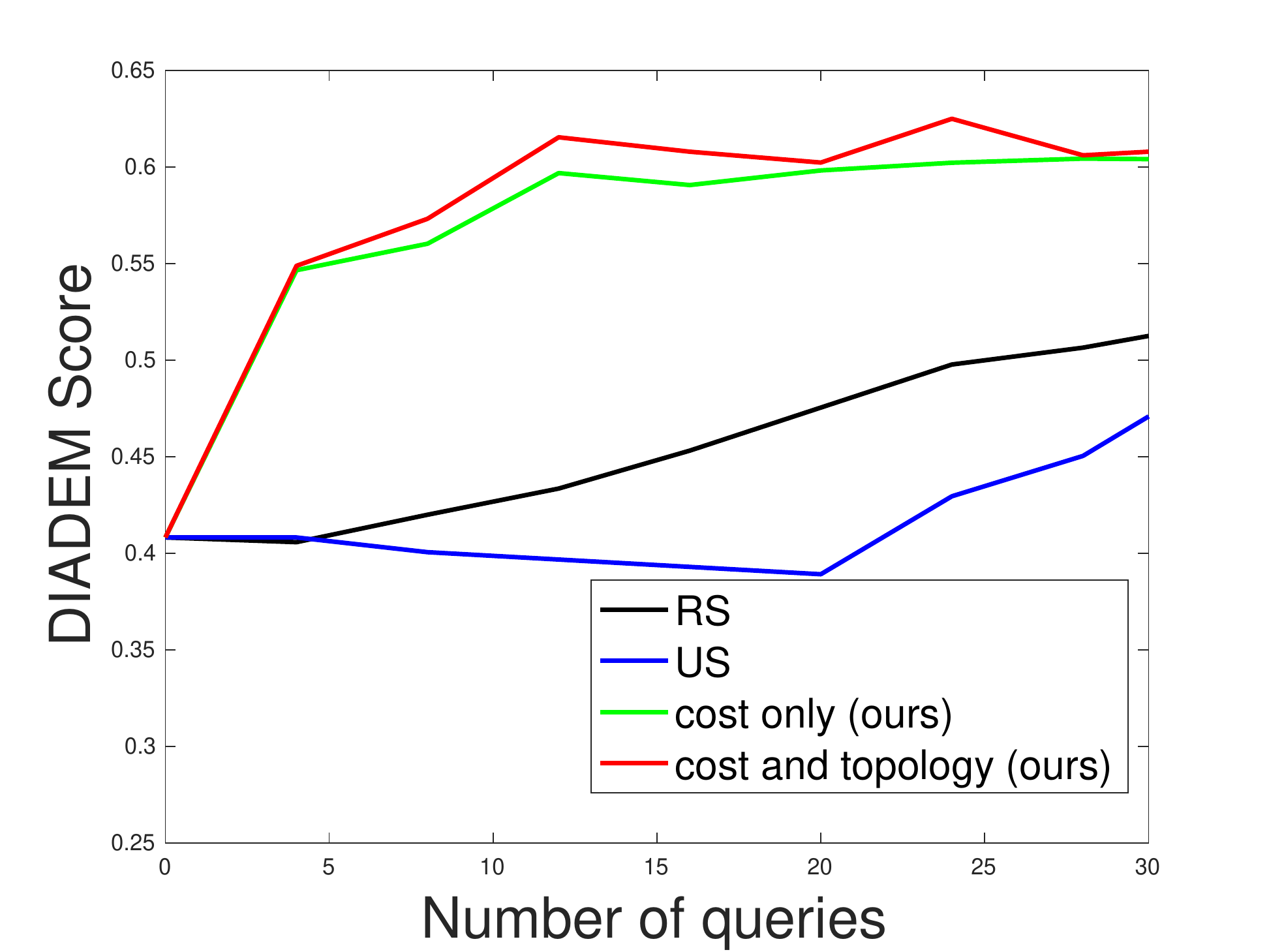}}
\subfloat[]{\includegraphics[height=0.33\textwidth]{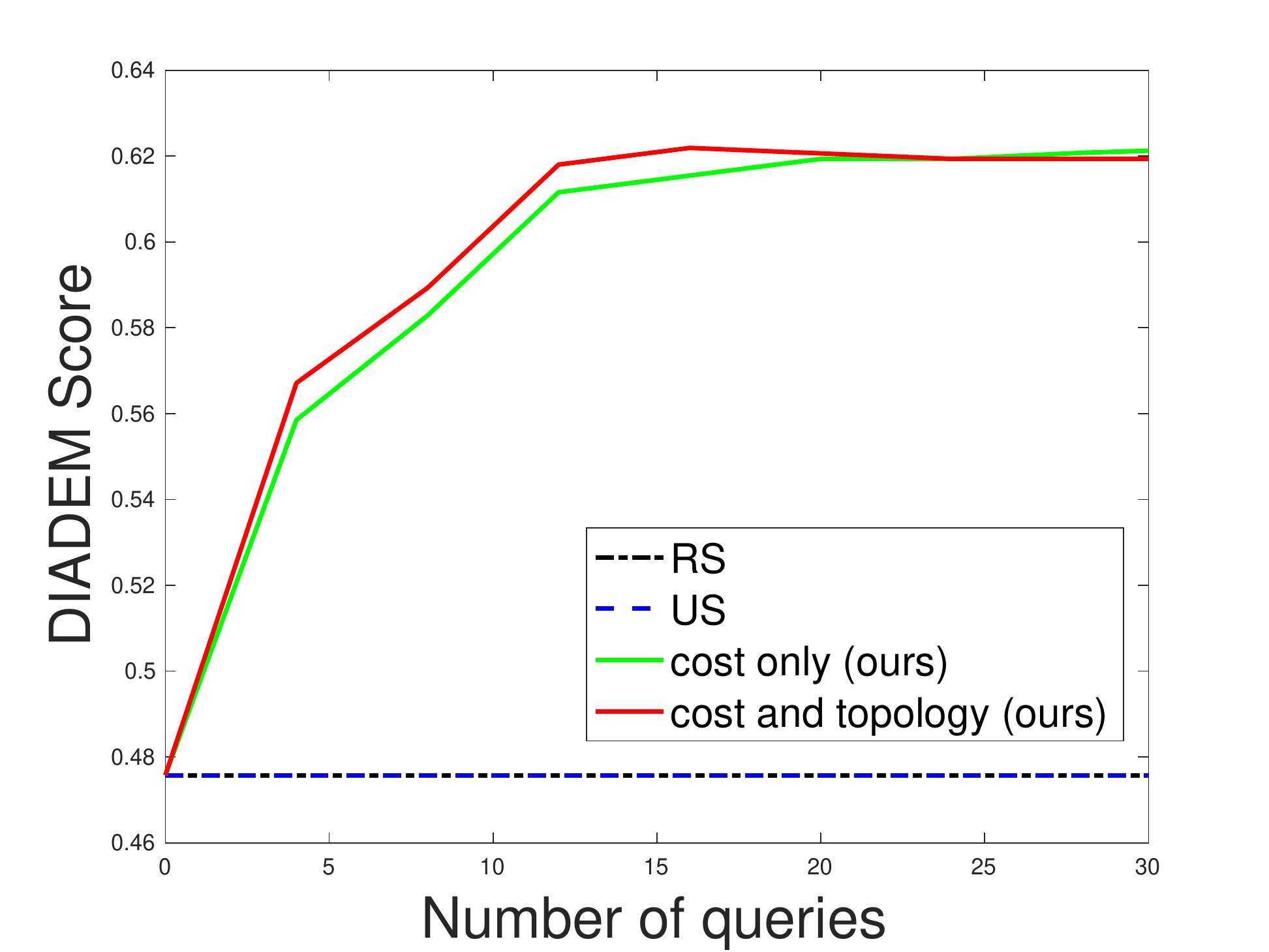}} \\ \vspace{-0.4cm}
\subfloat[]{\includegraphics[height=0.33\textwidth]{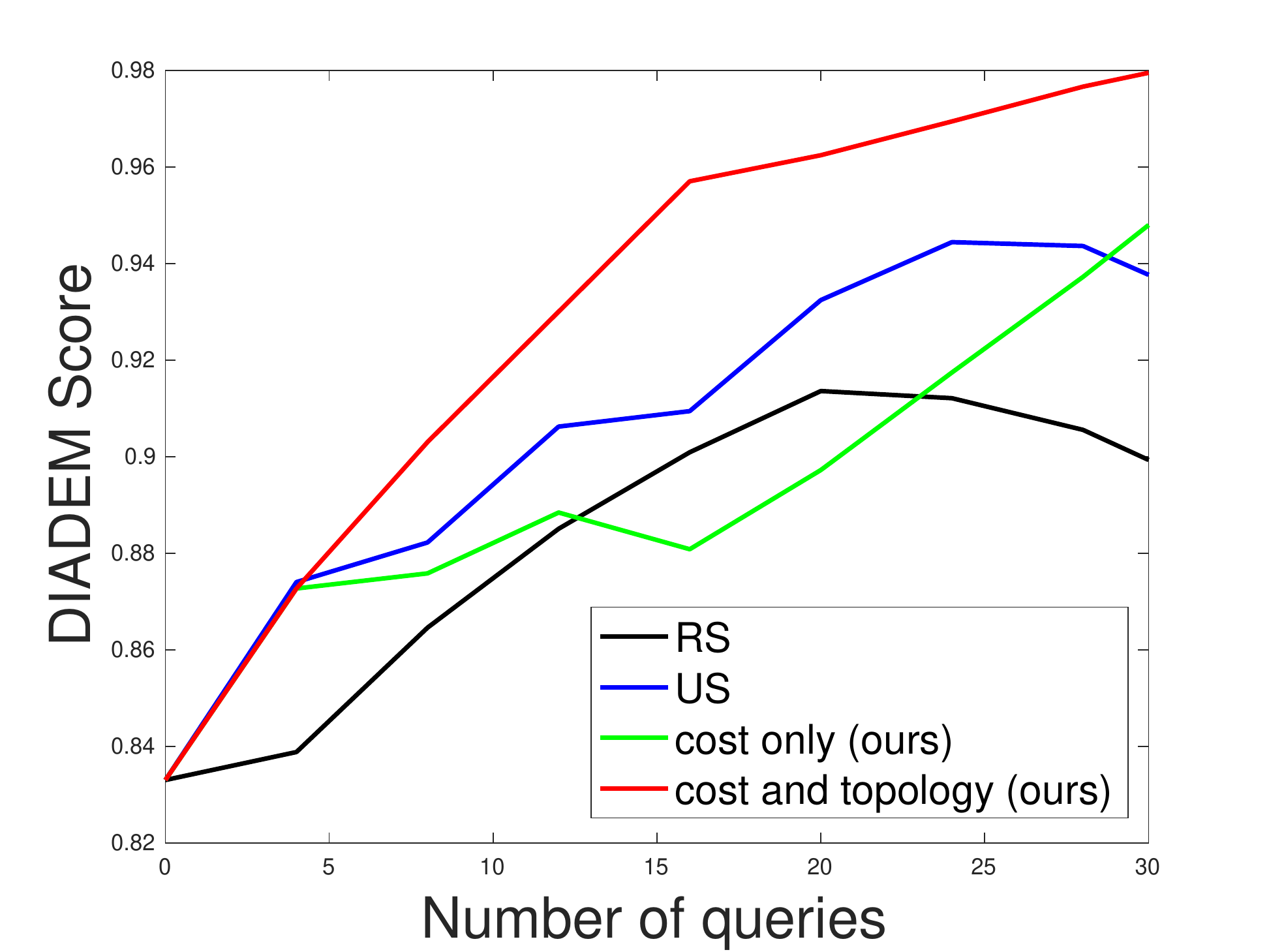}}
\subfloat[]{\includegraphics[height=0.33\textwidth]{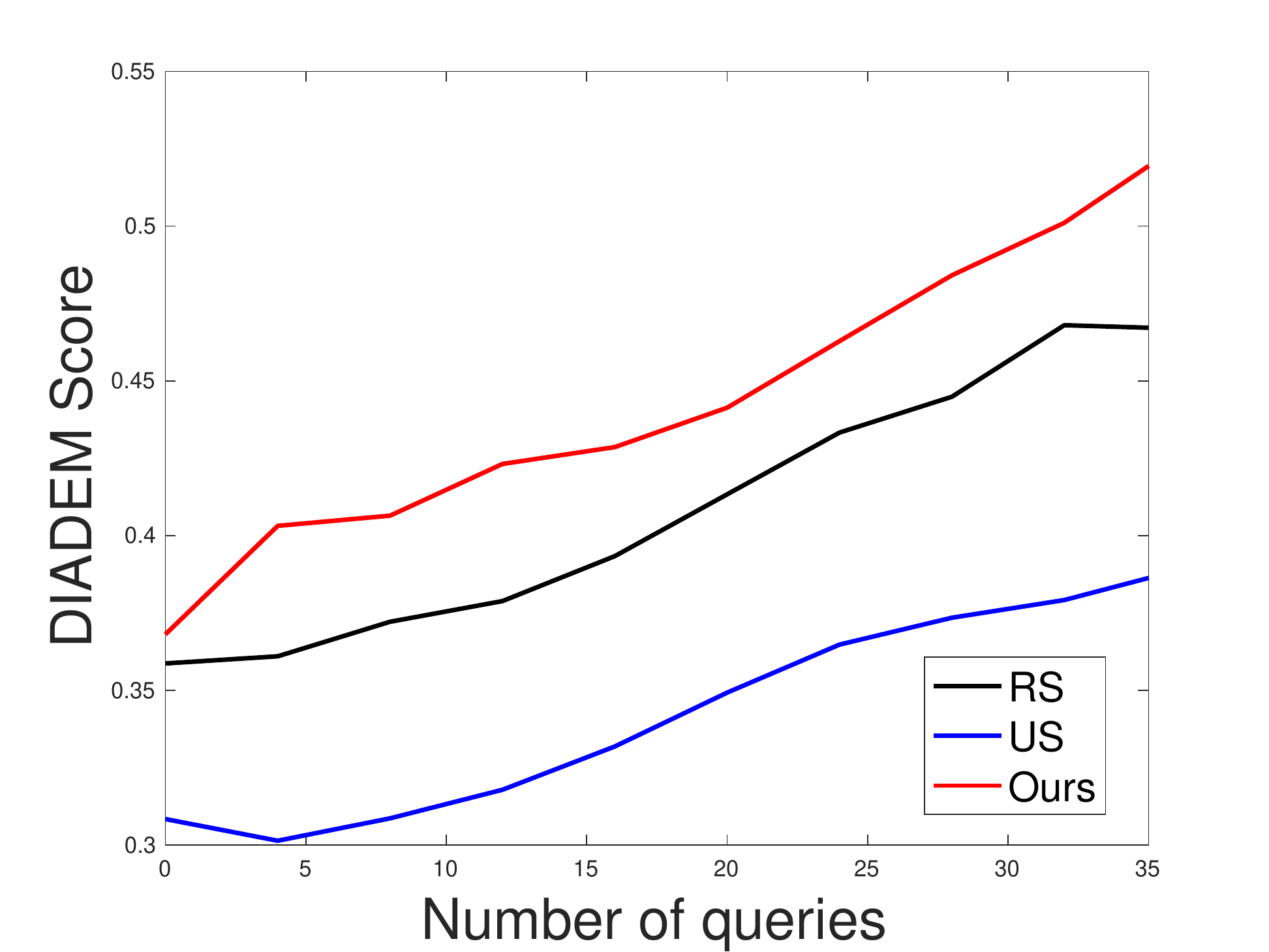}}
\caption{Focused  proofreading.  DIADEM  score as  a function  of the  number of
  paths examined  by the  annotator. (a) \textit{Axons.}  (b) \textit{Brightfield
    Neuron}.  (c)  \textit{Olfactory Projection  Fibers}.  (d)  Combined AL  and
  proofreading for \textit{Axons}.}
\label{fig:proof}
\end{figure*}

\begin{figure}
\centering
\includegraphics[width=0.99\textwidth]{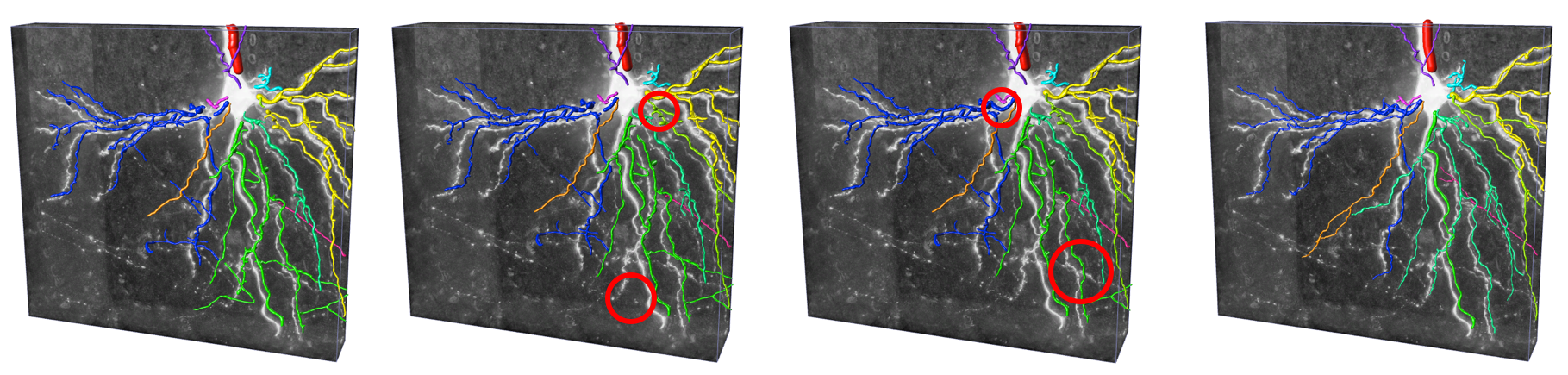}
\caption{Proofreading.  From  left to  right: initial  delineation, delineations
  after 10 and 20 corrections, and ground truth.}
\label{fig:proofSeq}
\end{figure}

For each  test image,  we compute  an overcomplete graph  and classify  its edges
using a classifier trained  on 20000 samples.  We then find four edges with
the highest values of the score $s_i$ of Eq.~\ref{eq:diadem} and present them to
the  user  for  verification.   Their  feedback   is  then  used  to  update  the
delineation.

The red curves of Fig.~\ref{fig:proof}(a-c) depict the increase in DIADEM score.
Rapid   improvement  can   be  seen   after   as  few   as  15   corrections.
Fig.~\ref{fig:proofSeq}  shows how  the reconstruction  evolves in  a specific
case. For analysis  purposes, we also reran the experiment  using the $\Delta c$
criterion  of   Eq.~\ref{eq:weightMetric}  (cost-only)   instead  of   the  more
sophisticated one of Eq.~\ref{eq:diadem} (cost and topology) to choose the paths
to  be  examined.  The  green  curves  in Fig.~\ref{fig:proof}(a-c)  depict  the
results.   They   are  not as good,   particularly  in  the   case  of
Fig.~\ref{fig:proof}(c), because the highest-scoring mistakes are often the ones
that tend to  be in the \QMIP{} reconstruction both  before and after correcting
mistakes.  It  is therefore  only by  combining both cost  and topology  that we
increase the chances that a potential  correction of the selected edge will improve the reconstruction. By
  contrast, paths chosen by \RS{} and \US{} are not necessarily erroneous or in
  the immediate neighborhood of the tree. As a result, investigating them often does
  not give any improvements.

\subsection{Complete Pipeline}
\label{sec:complete}

 In  a working  system, we  would integrate  AL and  proofreading into  a single
 pipeline.  To gauge its potential efficiency, we selected 50 edges to train our
 classifier  using the  AL strategy  of Section~\ref{sec:active}.  We then  computed a delineation in a test  image and proofread it by selecting  35 edges.  For comparison purposes, we
 used either our  approach as described in Section~\ref{sec:active}, \RS{}, or \US{} to
 pick  the  edges for training and then for verification.   In  Fig.~\ref{fig:proof}(d) we  plot  the
 performance  (in terms  of the  DIADEM score  of the  final delineation  and of the
 ground truth) as a function of the total  number of edges the user needed to
 label manually. 
 


\section{Conclusions}\label{sec:conclusion}

We have  presented an attention  scheme that significantly reduces  the annotation
effort involved both in creating training data for supervised Machine Learning and
in proofreading  results for delineation tasks.   It does so by  detecting possibly
misclassified samples  and considering  their influence on  the topology  of the
reconstruction. We  showed  that  our  method  outperforms  baselines  on  a variety of
microscopy image stacks and can be used in interactive applications thanks to its efficient formulation.

\bibliographystyle{alphaurl}
\bibliography{short,learning,biomed,vision}
\end{document}